\documentclass{article}
\usepackage{ijcai20}

\usepackage{times}
\usepackage{soul}
\usepackage{url}
\usepackage[hidelinks]{hyperref}
\usepackage[utf8]{inputenc}
\usepackage[small]{caption}
\usepackage{graphicx}
\usepackage{amsmath}
\usepackage{amsthm}
\usepackage{booktabs}
\usepackage{algorithm}
\usepackage{graphicx}
\usepackage{mathrsfs}
\graphicspath{{./images/}}
\usepackage[ddmmyyyy]{datetime}
\usepackage{helvet}
\usepackage{courier}
\usepackage{amsfonts}
\usepackage{tabu}
\usepackage{multirow}
\usepackage{verbatim}
\usepackage[noend]{algpseudocode}
\usepackage{subcaption}

\newtheorem{lemma}{Lemma}

\title{DEDPUL: Difference-of-Estimated-Densities-based Positive-Unlabeled Learning}

\author{
Dmitry Ivanov\\
\affiliations
National Research University Higher School of Economics\\
JetBrains Research\\
\emails
diivanov@hse.ru
}

\begin{document}

\maketitle

\begin{abstract}
   Positive-Unlabeled (PU) learning is an analog to supervised binary classification for the case when only the positive sample is clean, while the negative sample is contaminated with latent instances of positive class and hence can be considered as an unlabeled mixture. The objectives are to classify the unlabeled sample and train an unbiased PN classifier, which generally requires to identify the mixing proportions of positives and negatives first. Recently, unbiased risk estimation framework has achieved state-of-the-art performance in PU learning. This approach, however, exhibits two major bottlenecks. First, the mixing proportions are assumed to be identified, i.e. known in the domain or estimated with additional methods. Second, the approach relies on the classifier being a neural network. In this paper, we propose DEDPUL, a method that solves PU Learning without the aforementioned issues.\footnote{Implementation of DEDPUL is available at \url{https://github.com/dimonenka/DEDPUL}} The mechanism behind DEDPUL is to apply a computationally cheap post-processing procedure to the predictions of any classifier trained to distinguish positive and unlabeled data. Instead of assuming the proportions to be identified, DEDPUL estimates them alongside with classifying unlabeled sample. Experiments show that DEDPUL outperforms the current state-of-the-art in both proportion estimation and PU Classification.
\end{abstract}


\section{Introduction}\label{section_intro}


PU Classification naturally emerges in numerous real-world cases where obtaining labeled data from both classes is complicated. We informally divide the applications into three categories. In the first category, PU learning is an analog to binary classification. Here, the latter could be applied, but the former accounts for the label noise and hence is more accurate. An example is identification of disease genes. Typically, the already identified disease genes are treated as positive, and the rest unknown genes are treated as negative. Instead, \cite{yang2012desease} more accurately treat the unknown genes as unlabeled.

In the second category, PU learning is an analog to one-class classification, which learns only from positive data while making assumptions about the negative distribution. For instance, \cite{blanchard2010} show that density level set estimation methods \cite{vert2006consistency,scott2006learning} assume the negative distribution to be uniform, OC-SVM \cite{ocsvm} finds decision boundary against the origin, and OC-CNN \cite{occnn} explicitly models the negative distribution as a Gaussian. Instead, PU learning algorithms approximate the negative distribution as the difference between the unlabeled and the positive, which can be especially crucial when the distributions overlap significantly \cite{scott2009novelty}. Moreover, the outcome of the one-class approach is usually an anomaly metric, whereas PU learning can output unbiased probabilities. Examples are the tasks of anomaly detection, e.g. deceptive reviews \cite{ren2014reviews} or fraud \cite{nguyen2011ts}. 

The third category contains the cases for which neither supervised nor one-class classification could be applied. An example is identification of corruption in Russian procurement auctions \cite{ivanov2019}. An extensive data set of the auctions is available online, however it does not contain any labels of corruption. The proposed solution is to detect only successful corruptioneers by treating runner-ups as fair (positive) and winners as possibly corrupted (unlabeled) participants. PU learning may then detect suspicious patterns based on the subtle differences between the winners and the runner-ups, a task for which one-class approach is too insensitive.

\begin{figure*}[t]
    \includegraphics[width=\textwidth]{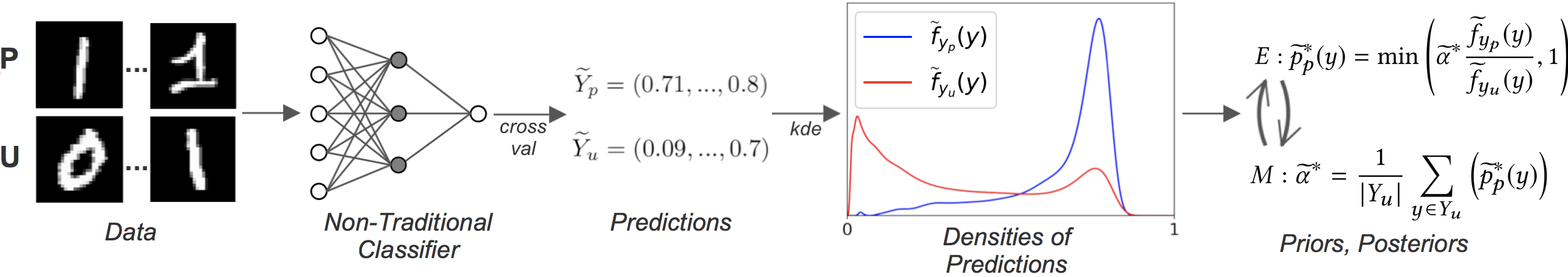}
    \caption{DEDPUL. \textit{Data:} Upper row (ones) represents positive sample $X_p$, lower row (ones and zeros) represents unlabeled sample $X_u$. \textit{Non-Traditional Classifier:} NTC is a supervised probabilistic binary classifier trained to distinguish $X_p$ from $X_u$. \textit{Predictions:} Predicted by NTC probabilities $Y_p$ and $Y_u$ for samples $X_p$ and $X_u$, obtained with cross-validation. NTC reduces dimensionality. \textit{Densities of Predictions:} Estimated probability density functions $\widetilde{f}_{y_p}(y)$ and $\widetilde{f}_{y_u}(y)$ on $Y_p$ and $Y_u$ respectively. $\widetilde{f}_{y_u}(y)$ has two distinct peaks, one of which coincides with that of $\widetilde{f}_{y_p}(y)$, while the other corresponds to the negative component of the mixture. \textit{Priors, Posteriors:} The posteriors are estimated for fixed priors using the Bayes rule and are clipped. The priors are estimated as the average posteriors. These two steps iterate until convergence. If the process converges trivially to 0, an alternative estimate $\widetilde{\alpha}_n^*$ is used instead.}
    \label{fig_dedpul}
\end{figure*}

One of the early milestones of PU classification is the paper of \cite{EN}. The paper makes two major contributions. The first contribution is the notion of Non-Traditional Classifier (NTC), which is any classifier trained to distinguish positive and unlabeled samples. Under the Selected Completely At Random (SCAR) assumption, which states that the labeling probability of any positive instance is constant, the biased predictions of NTC can be transformed into the unbiased posterior probabilities of being positive rather than negative. The second contribution is to consider the unlabeled data as simultaneously positive and negative, weighted by the opposite weights. By substituting these weights into the loss function, a PN classifier can be learned directly from PU data. These weights are equal to the posterior probabilities of being positive, estimated based on the predictions of NTC. The idea of loss function reweighting would later be adopted by the Risk Estimation framework \cite{plessis2014analysis,plessis2015convex} and its latest non-negative modification nnPU \cite{nnRE}, which is currently considered state-of-the-art. The crucial insight of this framework is that only the mixing proportion, i.e. the prior probability of being positive, is required to train an unbiased PN classifier on PU data.

While nnPU does not address mixture proportion estimation, multiple studies explicitly focus on this problem \cite{scott2014,scott2015rate,plessis2015prior}. The state-of-the-art methods are KM \cite{KM}, which is based on mean embeddings of positive and unlabeled samples into reproducing kernel Hilbert space, and TIcE \cite{davis2018tree}, which tries to find a subset of mostly positive data via decision tree induction. Another noteworthy method is AlphaMax \cite{alphamax}, which explicitly estimates the total likelihood of the data as a function of the proportions and seeks for the point where the likelihood starts to change

This paper makes several distinct contributions. First, we prove that 
NTC is a posterior-preserving transformation (Section \ref{section_algorithm_ntc}), whereas NTC is only known in the literature to preserve the priors \cite{alphamax}. This enables practical estimation of the posteriors using probability density functions of NTC predictions. Independently, we propose two alternative estimates of the mixing proportions (Section \ref{section_algorithm_alpha}). These estimates are based on the relation between the priors and the average posteriors. Finally, these new ideas are merged in DEDPUL (Difference-of-Estimated-Densities-based Positive-Unlabeled Learning), a method that simultaneously solves both Mixture Proportions Estimation and PU Classification, outperforms state-of-the-art methods for both problems \cite{KM,davis2018tree,nnRE}, and is applicable to a wide range of mixture proportions and data sets (Section \ref{section_results}).

DEDPUL adheres to the following two-step strategy. At the first step, NTC is constructed \cite{EN}. Since NTC treats unlabeled sample as negative, its predictions can be considered as biased posterior probabilities of being positive rather than negative. The second step should eliminate this bias. At the second step, the probability density functions of the NTC predictions for both positive and unalbeled samples are explicitly estimated. Then, both the prior and the posterior probabilities of being positive are estimated simultaneously by plugging these densities into the Bayes rule and performing Expectation-Maximization (EM). In case if EM converges trivially, an alternative algorithm is applied.

\section{Problem Setup and Background}\label{section_background}

In this section, we introduce the notations and formally define the problems of Mixture Proportions Estimation and PU Classification. Additionally, we discuss some of the common assumptions that we use in this work in Section \ref{section_background_proportions}.

\subsection{Preliminaries}\label{section_background_prelim}

Let random variables $x_p$, $x_n$, $x_u$ have probability density functions $f_{x_p}(x)$, $f_{x_n}(x)$, $f_{x_u}(x)$ and correspond to positive, negative, and unlabeled distributions of $x$, where $x \in \mathbb{R}^{m}$ is a vector of features. Let $\alpha$ be the proportion of $f_{x_p}(x)$ in $f_{x_u}(x)$:

\begin{equation}\label{eq_mixture}
    f_{x_u}(x) \equiv \alpha f_{x_p}(x) + (1 - \alpha) f_{x_n}(x)
\end{equation}


Let $p_p(x)$ be the posterior probability that $x \sim f_{x_u}(x)$ is sampled from $f_{x_p}(x)$ rather than $f_{x_n}(x)$. Given the proportion $\alpha$, it can be computed using the Bayes rule:

\begin{equation}\label{eq_poster}
    p_p(x) \equiv \frac{\alpha f_{x_p}(x)}{\alpha f_{x_p}(x) + (1 - \alpha) f_{x_n}(x)} = \alpha \frac{f_{x_p}(x)}{f_{x_u}(x)}
\end{equation}

\begin{algorithm*}[t]
\caption{DEDPUL}\label{DEDPUL}
\label{alg_dedpul}
\begin{algorithmic}[1]
\State{\textbf{Input:} Unlabeled sample $X_u$, Positive sample $X_p$}
\State{\textbf{Output:} Priors $\widetilde{\alpha}^*$, Posteriors \{$\widetilde{p}_p^*(x): x \in X_u$\}}

\State{$\widetilde{Y}_u, \widetilde{Y}_p = g(X_u), g(X_p) \backslash \backslash$ \emph{NTC predictions}}


\State{$\widetilde{f}_{y_u}(y), \widetilde{f}_{y_p}(y) = kde(\widetilde{Y}_u), kde(\widetilde{Y}_p) \backslash \backslash$ \emph{kernel density estimation}}


\State{$\widetilde{r}(y) = \{\frac{\widetilde{f}_{y_p}(y)}{\widetilde{f}_{y_u}(y)}: y \in \widetilde{Y}_u\}$ $\backslash \backslash$ \emph{array of density ratios, assumed to be sorted}} 
\State{$\widetilde{r}(y) = \textsc{monotonize}(\widetilde{r}(y), \widetilde{Y}_u) \backslash \backslash$ \emph{enforces partial monotonicity of the ratio curve, defined in Algorithm \ref{alg_secondary}}}
\State{$\widetilde{r}(y) = \textsc{rolling\_median}(\widetilde{r}(y)) \backslash \backslash$ \emph{smooths the ratio curve, defined in Algorithm \ref{alg_secondary}}}


\State{$\widetilde{\alpha}_c^*, \widetilde{p}_{p_c}^*(y) = \textsc{EM}(\widetilde{r}(y)); \quad \widetilde{\alpha}_n^*, \widetilde{p}_{p_n}^*(y)  = \textsc{max\_slope}(\widetilde{r}(y))\backslash \backslash$ \emph{estimates priors and posteriors, defined in Algorithm \ref{alg_secondary}}}
\State{\textbf{Return} $\widetilde{\alpha}_c^*, \widetilde{p}_{p_c}^*(y) \quad$ \textbf{if} $(\widetilde{\alpha}_c^* > 0) \quad$ \textbf{else} $\widetilde{\alpha}_n^*, \widetilde{p}_{p_n}^*(y)$}

\end{algorithmic}
\end{algorithm*}

\subsection{True Proportions are Unidentifiable}\label{section_background_proportions}

The goal of Mixture Proportions Estimation is to estimate the proportions of the mixing components in unlabeled data, given the samples $X_p$ and $X_u$ from the distributions $f_{x_p}(x)$ and $f_{x_u}(x)$ respectively. The problem is fundamentally ill-posed even if the distributions $f_{x_p}(x)$ and $f_{x_u}(x)$ are known. Indeed, based on these densities, a valid estimate of $\alpha$ is any $\widetilde{\alpha}$ (tilde denotes estimate) that fits the following constraint from (\ref{eq_mixture}):

\begin{equation}\label{eq_ineq_alpha_star}
    \forall x: f_{x_u}(x) \geq \widetilde{\alpha} f_{x_p}(x)
\end{equation}

In other words, the true proportion $\alpha$ is generally unidentifiable \cite{blanchard2010}, as it might be any value in the range $\alpha \in [0, \alpha^*]$. However, the upper bound $\alpha^*$ of the range can be identified directly from (\ref{eq_ineq_alpha_star}): 

\begin{equation}\label{eq_alpha_star}
    \alpha^* \equiv \underset{x \sim f_{x_u}(x)}{\inf} \frac{f_{x_u}(x)}{f_{x_p}(x)}
\end{equation}

\noindent Denote $p_p^*(x)$ as the corresponding to $\alpha^*$ posteriors:

\begin{equation}\label{eq_poster_star}
    p_p^*(x) \equiv \alpha^* \frac{f_{x_p}(x)}{f_{x_u}(x)}
\end{equation}

To cope with unidentifiability, a common practice is to make assumptions regarding the proportions, such as mutual irreducibility \cite{scott2013} or anchor set \cite{scott2015,KM}. Instead, we consider estimation of the identifiable upper-bound $\alpha^*$ as the objective of Mixture Proportions Estimation, explicitly acknowledging that the true proportion $\alpha$ might be lower. While the two approaches are similar, the distinction becomes meaningful during validation, in particular on synthetic data. In the first approach, the choice of data distributions is limited to the cases where the assumption holds. Otherwise, it might be unclear how to quantify the degree to which the assumption fails and thus its exact effect on misestimation of the proportions. In contrast, the upper-bound $\alpha^*$ can always be computed using equation (\ref{eq_alpha_star}) and taken as a target estimate. We employ this technique in our experiments.

Note that the probability that a positive instance $x$ is labeled is assumed to be constant. This has been formulated in the literature as Selected Completely At Random (SCAR) \cite{EN}. An emerging alternative is to allow the labeling probability $e(x)$ to depend on the instance, referred to as Selected At Random (SAR) \cite{davis2018beyond}. However, it is easy to see that the priors $\alpha^*$ and the labeling probability $e(x)$ are not uniquely identifiable under SAR without additional assumptions. Specifically, low $f_{x_u}(x)$ compared to $f_{x_p}(x)$ can be explained by either high $e(x)$ or low $\alpha^*$. Moreover, setting $\frac{e(x)}{1 - e(x)} = \frac{f_{x_p}(x)}{f_{x_u}(x)}$ leads to $\alpha^* = 1$ for any densities. Finding particular assumptions that relax SCAR while maintaining identifiability is an important problem of independent interest \cite{hammoudeh2020learning,shajarisales2020learning}. In this paper, however, we do not aim towards solving this problem. Instead, we propose a novel generic algorithm that assumes SCAR, but can be augmented with less restricting assumptions in the future work.

\subsection{Non-Traditional Classifier}\label{section_background_NTC}

Define Non-Traditional Classifier (NTC) as a function $g(x)$:

\begin{equation}\label{eq_ntc}
    g(x) \equiv \frac{f_{x_p}(x)}{f_{x_p}(x) + f_{x_u}(x)}
\end{equation}

\noindent In practice, NTC is a probabilistic classifier trained on the samples $X_p$ and $X_u$, balanced by reweighting the loss function. By definition (\ref{eq_ntc}), the proportions (\ref{eq_alpha_star}) and the posteriors (\ref{eq_poster_star}) can be estimated using $g(x)$:

\begin{equation}\label{eq_ntc_alpha_star}
    \alpha^* = \underset{x \sim f_{x_u}(x)}{\inf}\frac{1 - g(x)}{g(x)}
\end{equation}

\begin{equation}\label{eq_ntc_poster_star}
    p_p^*(x) = \alpha^* \frac{g(x)}{1 - g(x)}
\end{equation}

\noindent Directly applying (\ref{eq_ntc_alpha_star}) and (\ref{eq_ntc_poster_star}) to the output of NTC has been considered by \cite{EN} and is referred to as EN method. Its performance is reported in Section \ref{section_results_ablation}.

Let random variables $y_p$, $y_n$, $y_u$ have probability density functions $f_{y_p}(y)$, $f_{y_n}(y)$, $f_{y_u}(y)$ and correspond to the distributions of $y = g(x)$, where $x \sim f_{x_p}(x)$, $x \sim f_{x_n}(x)$, and $x \sim f_{x_u}(x)$ respectively.

\section{Algorithm Development}\label{section_algorithm}

In this section, we propose DEDPUL. The method is summarized in Algorithm \ref{alg_dedpul} and is illustrated in Figure \ref{fig_dedpul}, while the secondary functions are presented in Algorithm \ref{alg_secondary}. 

In the previous section, we have discussed the case of explicitly known distributions $f_{x_p}(x)$ and $f_{x_u}(x)$. However, only the samples $X_p$ and $X_u$ are usually available. Can we use these samples to estimate the densities $f_{x_p}(x)$ and $f_{x_u}(x)$ and still apply (\ref{eq_alpha_star}) and (\ref{eq_poster_star})? Formally, the answer is positive. For proportion estimation, this idea is explored in \cite{alphamax} as a baseline. In practice, two crucial issues may arise: the curse of dimensionality and the instability of $\alpha^*$ estimation. We attempt to fix these issues in DEDPUL.

\subsection{NTC as Dimensionality Reduction}\label{section_algorithm_ntc}

The first issue is that density estimation is difficult in high dimensions \cite{liu2007curse} and thus the dimensionality should be reduced. To this end, we propose the NTC transformation (\ref{eq_ntc}) that reduces the arbitrary number of dimensions to one. Below we prove that it preserves the priors $\alpha^*$:

\begin{equation}\label{eq_ntc_preserves_alpha}
    \alpha^* \equiv \underset{x \sim f_{x_u}(x)}{\inf} \frac{f_{x_u}(x)}{f_{x_p}(x)} = \underset{y \sim f_{y_u}(y)}{\inf} \frac{f_{y_u}(y)}{f_{y_p}(y)}
\end{equation}

\noindent and the posteriors:

\begin{equation}\label{eq_ntc_preserves_poster}
    p_p^*(x) \equiv \alpha^* \frac{f_{x_p}(x)}{f_{x_u}(x)} = \alpha^* \frac{f_{y_p}(y)}{f_{y_u}(y)} \equiv p_p^*(y)
\end{equation}

\begin{lemma}\label{lemma_NTC}
NTC is a posterior- and $\alpha^*$-preserving transformation, i.e. (\ref{eq_ntc_preserves_alpha}) and (\ref{eq_ntc_preserves_poster}) hold.\footnote{While (\ref{eq_ntc_preserves_alpha}) is known in the literature \cite{alphamax}, (\ref{eq_ntc_preserves_poster}) and the property that $\frac{f_{x_p}(x)}{f_{x_u}(x)} = \frac{f_{y_p}(y)}{f_{y_u}(y)}$ have not been proven before.}
\end{lemma}

\begin{proof}
To prove (\ref{eq_ntc_preserves_alpha}) and (\ref{eq_ntc_preserves_poster}), we only need to show that $\frac{f_{x_p}(x)}{f_{x_u}(x)} = \frac{f_{y_p}(y)}{f_{y_u}(y)}$, where $y = g(x)$. We will rely on the rule of variable substitution $f_{y_p}(y) = \sum_{x_r \in r} \frac{f_{x_p}(x_r)}{\nabla g(x_r)}$, where $r = \{x \mid g(x) = y\}$ or, equivalently, $r = \{x \mid \frac{f_{x_p}(x)}{f_{x_u}(x)} = \frac{y}{1-y}\}$:

\begin{equation*}
\begin{split}
    \frac{f_{y_p}(y)}{f_{y_u}(y)} = \frac{\sum_{x_r \in r} \frac{f_{x_p}(x_r)}{\nabla g(x_r)}}{\sum_{x_r \in r} \frac{f_{x_u}(x_r)}{\nabla g(x_r)}} = \frac{\sum_{x_r \in r} \frac{f_{x_u}(x_r)}{\nabla g(x_r)} \frac{f_{x_p}(x_r)}{f_{x_u}(x_r)}}{\sum_{x_r \in r} \frac{f_{x_u}(x_r)}{\nabla g(x_r)}} \\
    = \frac{y}{1-y} \frac{\sum_{x_r \in r} \frac{f_{x_u}(x_r)}{\nabla g(x_r)}}{\sum_{x_r \in r} \frac{f_{x_u}(x_r)}{\nabla g(x_r)}} =  \frac{y}{1-y} = \frac{f_{x_p}(x_r)}{f_{x_u}(x_r)}, x_r \in r
\end{split}
\end{equation*}
\end{proof}

\textit{Remark on Lemma \ref{lemma_NTC}}. In the proof, we only use the property of NTC that the density ratio $\frac{f_{x_p}(x)}{f_{x_u}(x)}$ is constant for all $x$ such that $g(x) = y$. Generally, any function $g(x)$ with this property preserves the priors and the posteriors. For instance, NTC could be biased towards 0.5 probability estimate due to regularization. In this case (\ref{eq_ntc_preserves_alpha}) and (\ref{eq_ntc_preserves_poster}) would still hold, whereas (\ref{eq_ntc_alpha_star}) and (\ref{eq_ntc_poster_star}) used in EN would yield incorrect estimates since they rely on NTC being unbiased.

Note that the choice of NTC is flexible. For example, NTC and its hyperparameters can be chosen by maximizing ROC-AUC, a metric that is invariant to label-dependent contamination \cite{menon2015learning,menon2016learning}.

After the densities of NTC predictions are estimated, two smoothing heuristics (Algorithm \ref{alg_secondary}) are applied to their ratio.

\begin{algorithm}[t]
\caption{Secondary functions for DEDPUL}
\label{alg_secondary}
\begin{algorithmic}[1]

\Function{EM}{$\widetilde{r}(y), tol=10^{-5}$}
\State{$\widetilde{\alpha}^* = 1$ $\backslash \backslash$ \emph{Initialize}}
\Repeat
    \State{$\widetilde{\alpha}_{prev}^* = \widetilde{\alpha}^*$}
    \State{$\widetilde{p}_p(y) = \min(\widetilde{\alpha}^* \cdot  \widetilde{r}(y), 1)$ $\backslash \backslash$ \emph{E-step}}
    \State{$\widetilde{\alpha}^* = \frac{1}{\left| Y_u \right|}\sum_{y \in Y_u}\left(\widetilde{p}_p(y)\right)$ $\backslash \backslash$ \emph{M-step}}
\Until \emph{$|\widetilde{\alpha}_{prev}^* - \widetilde{\alpha}^*| < tol$} $\backslash \backslash$ \emph{Convergence}
\State{\textbf{Return} $\widetilde{\alpha}^*$, $\widetilde{p}_p(y)$}
\EndFunction


\Function{max\_slope}{$\widetilde{r}(y), max\_D = 0.05, \epsilon = 10^{-3}$}
\State{$D = list()$ $\backslash \backslash$ \emph{Priors subtract average posteriors}}
\For{$\widetilde{\alpha}$ \textbf{in} range(start=0, end=1, step=$\epsilon$)}
    \State{$\widetilde{p}_p(y) = \min(\widetilde{\alpha} \cdot  \widetilde{r}(y), 1)$}
    \State{$idx = \widetilde{\alpha}/\epsilon$}
    \State{$D[idx] = \widetilde{\alpha} - \frac{1}{\left| Y_u \right|}\sum_{y \in Y_u}\left(\widetilde{p}_p(y)\right)$}
\EndFor
\State{$l = list()$ $\backslash \backslash$ \emph{Second lags of $D$}}
\For{i \textbf{in} range(start=1, end=$\frac{1}{\epsilon}-1$, step=1)}
\State{$l[i] = (D[i-1] - D[i]) - (D[i] - D[i+1])$}
\EndFor
\State{$\widetilde{\alpha}^* = argmax\left(l\right)$ s. t. $D < max\_D$}
\State{$\widetilde{p}_p(y) = \min(\widetilde{\alpha}^* \cdot \widetilde{r}(y), 1)$}
\State{\textbf{Return} $\widetilde{\alpha}^*$, $\widetilde{p}_p(y)$}
\EndFunction


\Function{monotonize}{$\widetilde{r}(y), \widetilde{Y}_u$}
\State{$max_{cur} = 0, threshold = \frac{1}{\left|\widetilde{Y}_u\right|}\sum_{\widetilde{y} \in \widetilde{Y}_u}\widetilde{y}$}
\For{i \textbf{in} range(start=0, end=$length(\widetilde{r}(y))$, step=1)}
\If{$\widetilde{Y}_u[i] \geq threshold$}
\State{$max_{cur} = \max(\widetilde{r}(y)[i], max_{cur})$}
\State{$\widetilde{r}(y)[i] = max_{cur}$}
\EndIf
\EndFor
\State{\textbf{Return} $\widetilde{r}(y)$}
\EndFunction


\Function{rolling\_median}{$\widetilde{r}(y)$}
\State{$L = length(\widetilde{r}(y)), k = L/20$}
\State{$\widetilde{r}_{new}(y) = \widetilde{r}(y)$}
\For{i \textbf{in} range(start=0, end=$L$, step=1)}
\State{$k_{cur} = \min(k, i, L - i)$ $\backslash \backslash$ \emph{Number of neighbours}}
\State{$\widetilde{r}_{new}(y)[i] = median(\widetilde{r}(y)[i-k_{cur}: i+k_{cur}])$}
\EndFor
\State{\textbf{Return} $\widetilde{r}_{new}(y)$}
\EndFunction

\end{algorithmic}
\end{algorithm}

\subsection{Alternative $\alpha^*$ Estimation}\label{section_algorithm_alpha}

The second issue is that (\ref{eq_ntc_preserves_alpha}) may systematically underestimate $\alpha^*$ as it solely relies on the noisy infimum point. This concern is confirmed experimentally in Section \ref{section_results_ablation}. To resolve it, we propose two alternative estimates. We first prove that these estimates coincide with $\alpha^*$ when the densities are known, and then formulate their practical approximations.

\begin{figure*}[ht]
    \centering
    \includegraphics[width=0.3\linewidth]{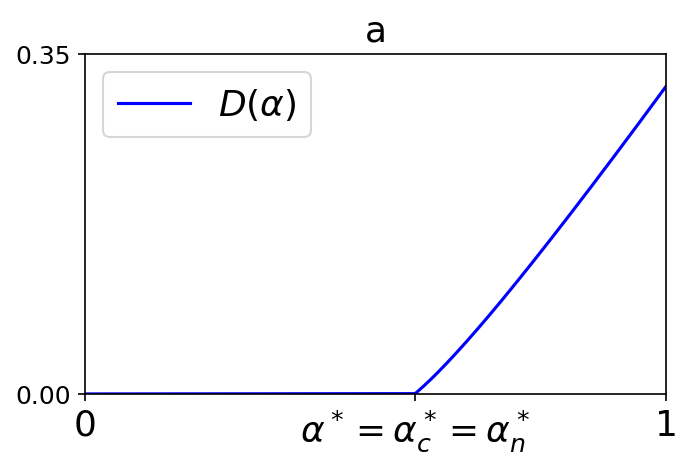}%
    \includegraphics[width=0.3\linewidth]{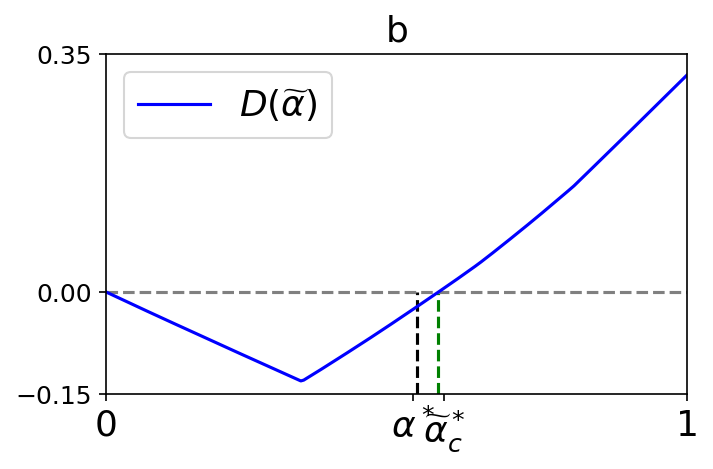}%
    \includegraphics[width=0.3\linewidth]{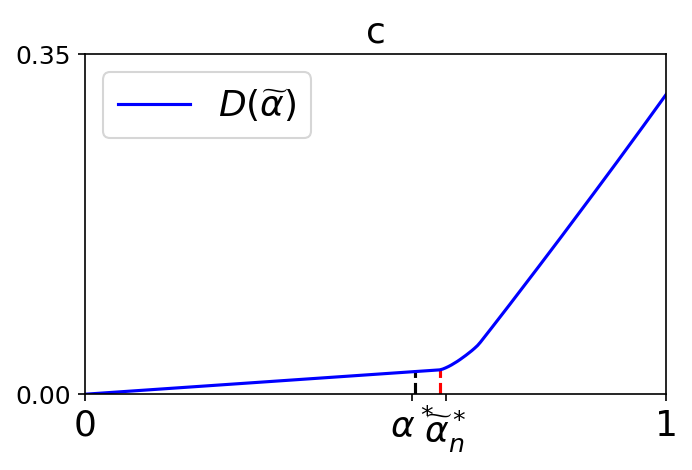}
    \caption{Behaviours of $D(\alpha)$ in theory (a) and practice (b, c). In these plots, $f_{x_p}(x) = L(0, 1)$, $f_{x_n} = L(2, 1)$, $\alpha = 0.5$, and $\alpha^* = 0.568$, where $L(\cdot, \cdot)$ denotes Laplace distribution.}
    \label{fig_D}
\end{figure*}

\begin{lemma}
    Alternative estimates of the priors are equal to the definition (\ref{eq_ntc_preserves_alpha}): $\alpha_c^* = \alpha_n^* = \alpha^*$, where \\ $D(\alpha) \equiv \alpha - \mathop{\mathbb{E}}_{y \sim f_{y_u}(y)}\left[p_p(y)\right]$, \\ $\alpha_c^* \equiv \underset{\alpha \in [0, 1]}{\max} \left\{ \alpha \mid D(\alpha) = 0 \right\}$, \\ $\alpha_n^* \equiv \underset{\alpha \in [0, 1]}{\min} \left\{ \alpha \mid D(\alpha + \epsilon) - 2 D(\alpha) + D(\alpha - \epsilon) > 0 \right\}$.
\end{lemma}

\begin{proof}
    We first prove that $D(\alpha)$ behaves in a specific way (Fig. \ref{fig_D}a). For the valid priors $\alpha \leq \alpha^*$, it holds that $D(\alpha) = 0$, since priors must be equal to expected posteriors: $p(h) = \mathbb{E}_{x \sim f(x)} p(h \mid x)$ for some hypothesis $h$ and data $x \sim f(x)$. 
    
    For the invalid priors $\alpha > \alpha^*$, the situation changes. By definition of $\alpha^*$ (\ref{eq_ntc_preserves_alpha}), it is now the case that $p_p(y) = \alpha \frac{f_{y_p}(y)}{f_{y_u}(y)} > 1$ for some $y$. To prevent this, the posteriors are clipped from above like in (\ref{eq_E}). Under this redefinition, the equality of priors to expected posteriors no longer holds. Instead, clipping decreases the posteriors in comparison to the priors. Thus, for $\alpha > \alpha^*$ it holds that $\alpha > \mathop{\mathbb{E}}_{y \sim f_{y_u}(y)}\left[p_p(y)\right]$ and $D(\alpha) > 0$. 
    
    So, $D(\alpha)$ equals 0 up until $\alpha^*$ and is positive after. It trivially follows that $\alpha^*$ is the highest $\alpha$ for which $D(\alpha) = 0$, and the lowest $\alpha$ for which $D(\alpha + \epsilon) - 2 \cdot D(\alpha) + D(\alpha - \epsilon) = D(\alpha + \epsilon) > 0$ for any $\epsilon$ (in practice, $\epsilon$ is chosen small).
\end{proof}

Figures \ref{fig_D}b and \ref{fig_D}c reflect how $D(\alpha)$ changes in practice due to imperfect density estimation: in simple words, the curve either `rotates` down or up. In the more frequent first case, non-trivial $\widetilde{\alpha}_c^*$ exists (the curve intersects zero). We identify it using Expectation-Maximization algorithm, but analogs like binary search could be used as well. In the second case (e.g. sometimes when $\alpha^*$ is extremely low), non-trivial $\widetilde{\alpha}_c^*$ does not exist and \textsc{EM} converges to $\widetilde{\alpha} = 0$. In this case, we instead approximate $\widetilde{\alpha}_n^*$ using \textsc{MAX\_SLOPE} function. We define these functions in Algorithm \ref{alg_secondary} and cover them below in more details. In DEDPUL, the two estimates are switched by choosing the maximal. In practice, we recommend to plot $D(\alpha)$ to visually identify $\alpha_c^*$ (if exists) or $\alpha_n^*$.

\paragraph{Approximation of $\alpha_c^*$ with \textsc{EM} algorithm.} On Expectation step, the proportion $\widetilde{\alpha}$ is fixed and the posteriors $\widetilde{p}_{p}(y)$ are estimated using clipped (10):

\begin{equation}\label{eq_E}
    \widetilde{p}_{p}(y)  = \min\left(\widetilde{\alpha} \frac{\widetilde{f}_{y_p}(y)}{\widetilde{f}_{y_u}(y)}, 1\right)
\end{equation}

\noindent On Maximization step, $\widetilde{\alpha}$ is updated as the average value of the posteriors over the unlabeled sample $\widetilde{Y}_u$:

\begin{equation}\label{eq_M}
    \widetilde{\alpha} = \frac{1}{\left| Y_u \right|}\sum_{y \in Y_u}\left(\widetilde{p}_{p}(y)\right)
\end{equation}

\noindent These two steps iterate until convergence. The algorithm should be initialized with $\widetilde{\alpha}=1$.

\paragraph{Approximation of $\alpha_n^*$ with \textsc{MAX\_SLOPE}.} On a grid of proportion values $\widetilde{\alpha} \in [0, 1]$, second difference of $D(\widetilde{\alpha})$ is calculated for each point. Then, the proportion that corresponds to the highest second difference is chosen.

\section{Experimental Procedure}\label{section_procedure}

We conduct experiments on synthetic and benchmark data sets. Although DEDPUL is able to solve both Mixture Proportions Estimation and PU Classification simultaneously, some of the algorithms are not. For this reason, all PU Classification algorithms receive $\alpha^*$ (synthetic data) or $\alpha$ (benchmark data) as an input. The algorithms are tested on numerous data sets that differ in the distributions and the proportions of the mixing components. Additionally, each experiment is repeated 10 times, and the results are averaged. 

The statistical significance of the difference between the algorithms is verified in two steps. First, the Friedman test with 0.02 P-value threshold is applied to check whether all the algorithms perform equivalently. If this null hypothesis is rejected, the algorithms are compared pair-wise using the paired Wilcoxon signed-rank test with 0.02 P-value threshold and Holm correction for multiple hypothesis testing. The results are summarized via critical difference diagrams (Fig. \ref{fig_cdd}), the code for which is taken from \cite{IsmailFawaz2018deep}.


\subsection{Data}

In the synthetic setting we experiment with mixtures of one-dimensional Laplace distributions. We fix $f_{x_p}(x)$ as $L(0, 1)$ and mix it with different $f_{x_n}(x)$: $L(1, 1)$, $L(2, 1)$, $L(4, 1)$, $L(0, 2)$, $L(0, 4)$. For each of these cases, the proportion $\alpha$ is varied in \{0.01, 0.05, 0.25, 0.5, 0.75, 0.95, 0.99\}. The sample sizes $X_p$ and $X_u$ are fixed as 1000 and 10000.

In the benchmark setting (Table \ref{table_data}) we experiment with eight data sets from UCI machine learning repository \cite{UCI}, MNIST image data set of handwritten digits \cite{MNIST}, and CIFAR-10 image data set of vehicles and animals \cite{cifar10}. The proportion $\alpha$ is varied in \{0.05, 0.25, 0.5, 0.75, 0.95\}. The methods' performance is mostly determined by the size of the labeled sample $\left| X_p \right|$. Therefore, for a given data set, $\left| X_p \right|$ is fixed across the different values of $\alpha$, and for a given $\alpha$, the maximal $\left| X_u \right|$ that can satisfy it is chosen. Categorical features are transformed using dummy encoding. Numerical features are normalized. Regression and multi-classification target variables are adapted for binary classification.

\begin{table}[ht]
\centering
\begin{tabular}{|c|c|c|c|c|}
\hline
data set  & \begin{tabular}[c]{@{}c@{}}total\\ size\end{tabular} & $\left|X_p\right|$ & dim  & \begin{tabular}[c]{@{}c@{}}positive\\ target\end{tabular}                                                  \\ \hline
bank      & 45211                                                & 1000               & 16   & yes                                                              \\ \hline
concrete  & 1030                                                 & 100                & 8    & (35.8, 82.6)                                                     \\ \hline
landsat   & 6435                                                 & 1000               & 36   & 4, 5, 7                                                          \\ \hline
mushroom  & 8124                                                 & 1000               & 22   & p                                                                \\ \hline
pageblock & 5473                                                 & 100                & 10   & 2, 3, 4, 5                                                       \\ \hline
shuttle   & 58000                                                & 1000               & 9    & 2, 3, 4, 5, 6, 7                                                 \\ \hline
spambase  & 4601                                                 & 400                & 57   & 1                                                                \\ \hline
wine      & 6497                                                 & 500                & 12   & red                                                              \\ \hline
mnist     & 70000                                                & 1000               & 784  & 1, 3, 5, 7, 9                                                    \\ \hline
cifar10   & 60000                                                & 1000               & 3072 & \begin{tabular}[c]{@{}c@{}}plane, car\\ ship, track\end{tabular} \\ \hline
\end{tabular}
\caption{Description of benchmark data sets}
\label{table_data}
\end{table}

\subsection{Measures for Performance Evaluation}


The synthetic setting provides a straightforward way to evaluate performance. Since the underlying distributions $f_{x_p}(x)$ and $f_{x_u}(x)$ are known, we calculate the true values of the proportions $\alpha^*$ and the posteriors $p_p^*(x)$ using (\ref{eq_alpha_star}) and (\ref{eq_poster_star}) respectively. Then, we directly compare these values with algorithm's estimates using mean absolute errors. In the benchmark setting, the distributions are unknown. Here, for Mixture Proportions Estimation we use a similar measure, but substitute $\alpha^*$ with $\alpha$, while for PU Classification we use $1 - accuracy$ with $0.5$ probability threshold. We additionally compare the algorithms using f1-measure, a popular metric for imbalanced data sets. 



\begin{figure*}[t!]
    \centering
    \begin{subfigure}{0.49\linewidth}
        \centering
        \includegraphics[width=\linewidth]{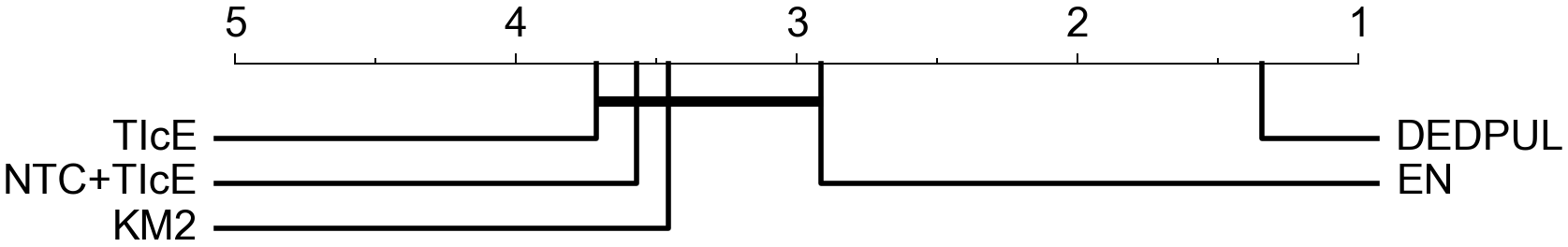}
        \caption{Synthetic data, $\left| \alpha^* - \widetilde{\alpha}^* \right|$}
    \end{subfigure}%
    ~ 
    \begin{subfigure}{0.49\linewidth}
        \centering
        \includegraphics[width=\linewidth]{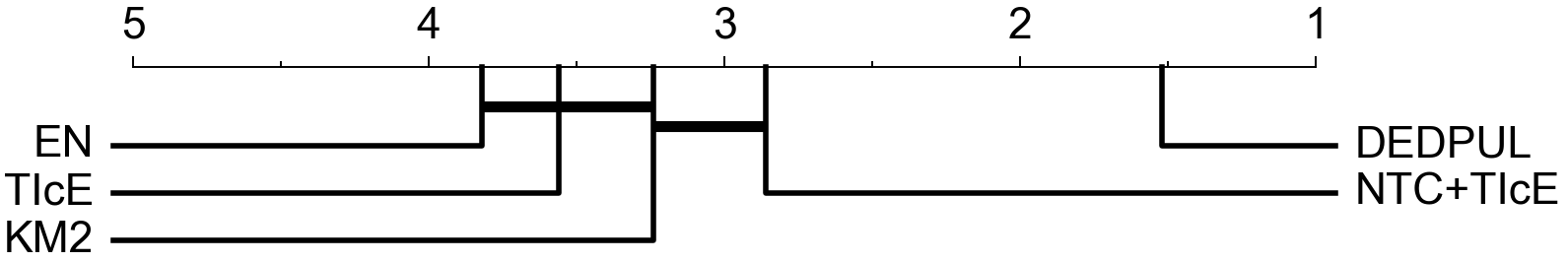}
        \caption{Benchmark data, $\left| \alpha - \widetilde{\alpha}^* \right|$}
    \end{subfigure}
    
    \begin{subfigure}{0.49\linewidth}
        \centering
        \includegraphics[width=\linewidth]{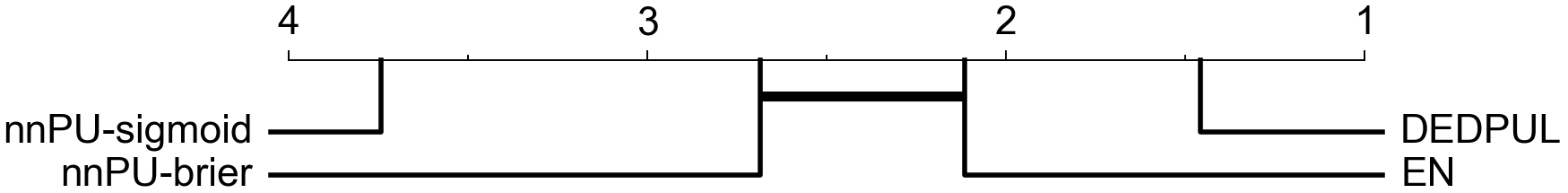}
        \caption{Synthetic data, $\left| p_p^*(x) - \widetilde{p}_p^*(x)\right|$}
    \end{subfigure}%
    ~
    \begin{subfigure}{0.49\linewidth}
        \centering
        \includegraphics[width=\linewidth]{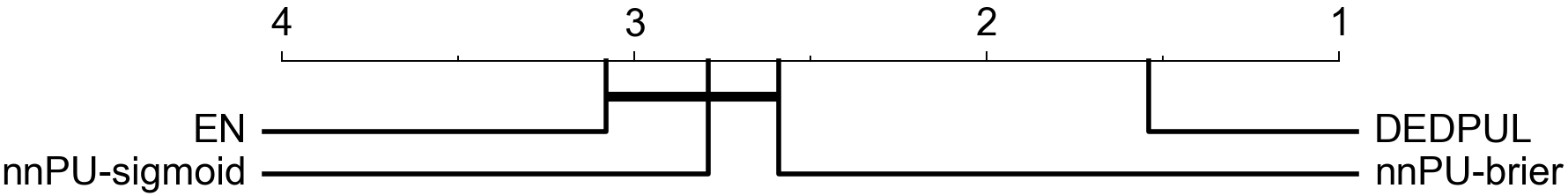}
        \caption{Benchmark data, 1 - accuracy}
    \end{subfigure}
    
    \begin{subfigure}{0.49\linewidth}
        \centering
        \includegraphics[width=\linewidth]{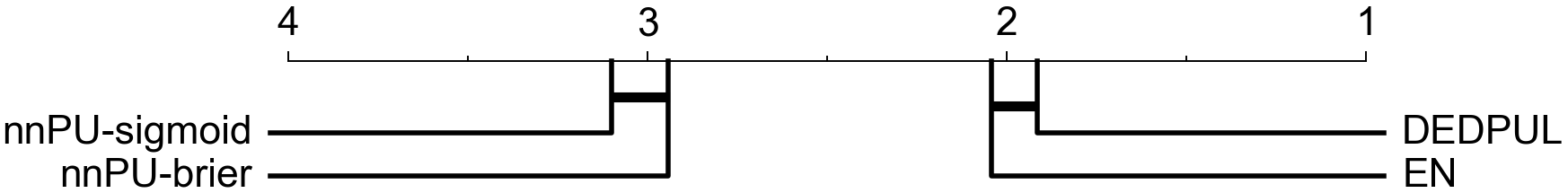}
        \caption{Synthetic data, 1 - f1-measure}
    \end{subfigure}%
    ~
    \begin{subfigure}{0.49\linewidth}
        \centering
        \includegraphics[width=\linewidth]{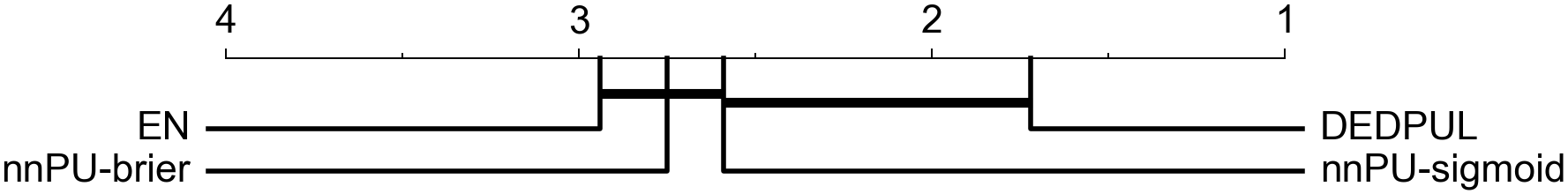}
        \caption{Benchmark data, 1 - f1-measure}
    \end{subfigure}

    \caption{Critical difference diagrams. Summarized comparison of the algorithms on synthetic and benchmark data with different performance measures. The horizontal axis represents relative ranking. Bold horizontal lines represent statistically insignificant differences. In the synthetic setup, accuracy differs insignificantly between the algorithms according to the Friedman test, so the diagram is not reported.}
    \label{fig_cdd}
\end{figure*}

\subsection{Implementations of Methods}

DEDPUL is implemented according to Algorithms \ref{alg_dedpul} and \ref{alg_secondary}. As NTC we apply neural networks with 1 hidden layer with 16 neurons (1-16-1) to the synthetic data, and with 2 hidden layers with 512 neurons (dim-512-512-1) and with batch normalization \cite{bn} after the hidden layers to the benchmark data. To CIFAR-10, a light version of all-convolutional net \cite{allconv} with 6 convolutional and 2 fully-connected layers is applied, with 2-d batch normalization after the convolutional layers. These architectures are used henceforth in EN and nnPU. The networks are trained on the logistic loss with ADAM \cite{adam} optimizer with 0.0001 weight decay. The predictions of each network are obtained with 5-fold cross-validation. 

Densities of the predictions are computed using kernel density estimation with Gaussian kernels. Instead of $\widetilde{y} \in [0, 1]$, we estimate densities of appropriately ranged $log\left(\frac{\widetilde{y}}{1 - \widetilde{y}}\right)$ and make according post-transformations. Bandwidths are set as 0.1 and 0.05 for $\widetilde{f}_{y_p}(y)$ and $\widetilde{f}_{y_u}(y)$ respectively. In practice, a well-performing heuristic is to independently choose bandwidths for $\widetilde{f}_{y_p}(y)$ and $\widetilde{f}_{y_u}(y)$ that maximize average log-likelihood during cross-validation. Threshold in \textsc{monotonize} is chosen as $\frac{1}{\left| X_u \right|}\sum_{x \in X_u}\left(\widetilde{g}(x)\right)$.

The implementations of the Kernel Mean based gradient thresholding algorithm (KM) are retrieved from the original paper \cite{KM}.\footnote{\url{http://web.eecs.umich.edu/~cscott/code.html##kmpe}.} We provide experimental results for KM2, a better performing version. As advised in the paper, MNIST and CIFAR-10 data are reduced to 50 dimensions with Principal Component Analysis \cite{pca}. Since KM2 cannot handle big data sets, it is applied to subsamples of at most 4000 examples ($X_p$ is prioritized).

The implementation of Tree Induction for label frequency $c$ Estimation (TIcE) algorithm is retrieved from the original paper \cite{davis2018tree}.\footnote{\url{https://dtai.cs.kuleuven.be/software/tice}} MNIST and CIFAR-10 data are reduced to 200 dimensions with Principal Component Analysis. In addition to the vanilla version, we also apply TIcE to the data preprocessed with NTC (NTC+TIcE). The performance of both versions heavily depends on the choice of the hyperparameters such as the probability lower bound $\delta$ and the number of splits per node $s$. These parameters were chosen as $\delta = 0.2$ and $s = 40$ for univariate data, and $\delta = 0.2$ and $s = 4$ for multivariate data.

We explore two versions of the non-negative PU learning (nnPU) algorithm \cite{nnRE}: the original version with Sigmoid loss function and our modified version with Brier loss function.\footnote{Sigmoid and Brier loss functions are mean absolute and mean squared errors between classifier's predictions and true labels.} The hyperparameters were chosen as $\beta = 0$ and $\gamma = 1$ for the benchmark data, and as $\beta = 0.1$ and $\gamma = 0.9$ for the synthetic data.

Elkan-Noto (EN) is implemented according to the original paper \cite{EN}. Out of the three proposed estimators of the labeling probability, we report results for $e_3$ (equivalent to (\ref{eq_ntc_alpha_star})) since it has shown better performance.

\section{Experimental results}\label{section_results}

\subsection{Comparison with State-of-the-art}\label{section_results_sota}

Results of comparison between DEPDUL and state-of-the-art are presented in Figures \ref{fig_mpe_synth}-\ref{fig_puc_bench} and are summarized in Figure \ref{fig_cdd}. DEDPUL significantly outperforms state-of-the-art of both problems in both setups for all measures, with the exception of insignificant difference between f1-measures of DEDPUL and nnPU-sigmoid in the benchmark setup.

\begin{figure*}[t]
\centering
  \includegraphics[width=\linewidth]{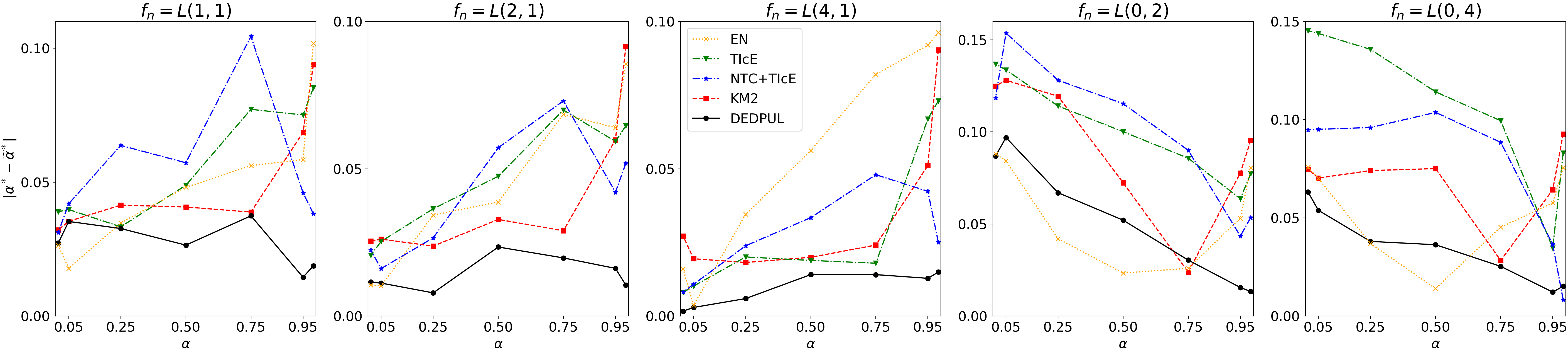}
  \caption{Comparison of Mixture Proportions Estimation Algorithms, Synthetic Data}
  \label{fig_mpe_synth}
  
  \includegraphics[width=\linewidth]{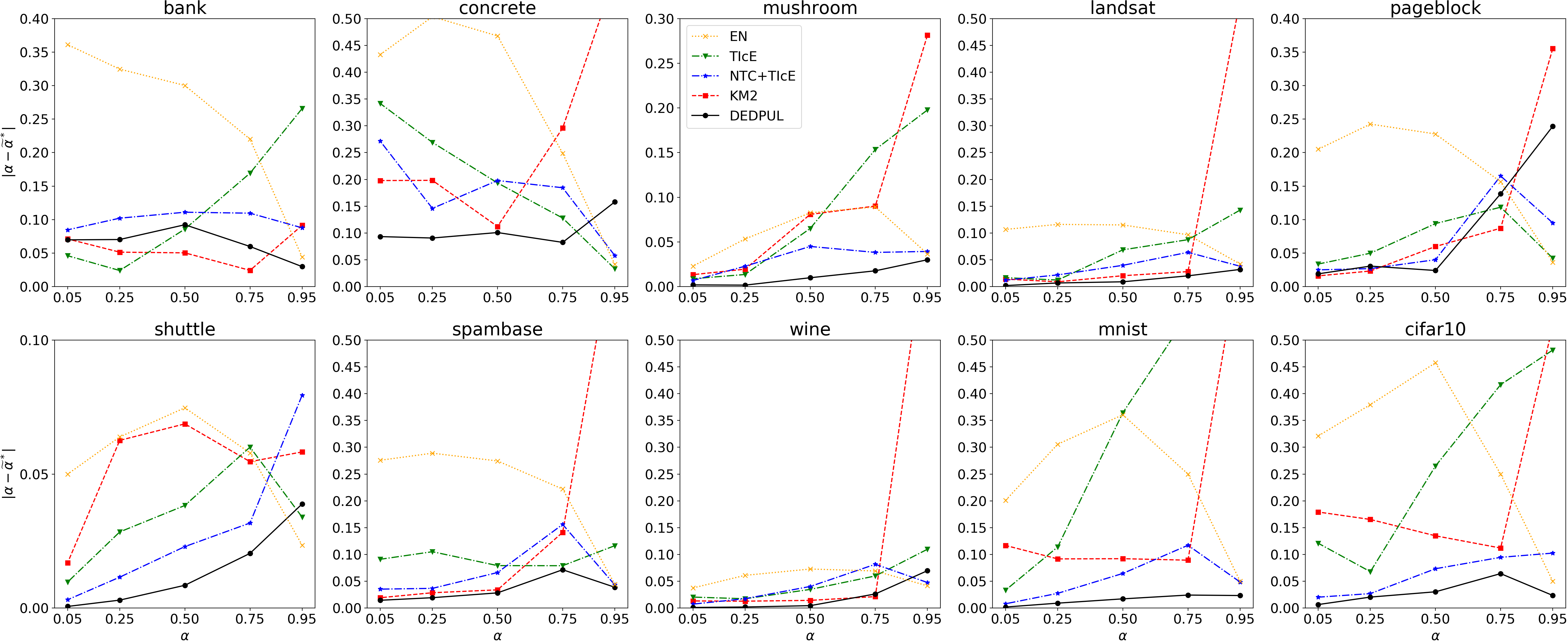}
  \caption{Comparison of Mixture Proportions Estimation Algorithms, Benchmark Data}
  \label{fig_mpe_bench}
  
  \includegraphics[width=\linewidth]{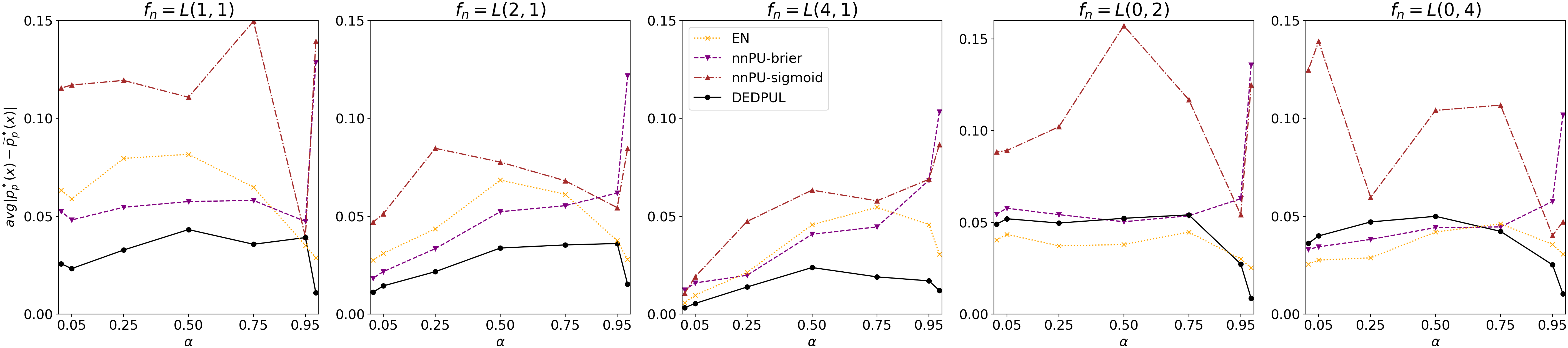}
  \caption{Comparison of Positive-Unlabeled Classification Algorithms, Synthetic Data}
  \label{fig_puc_synth}
  
  \includegraphics[width=\linewidth]{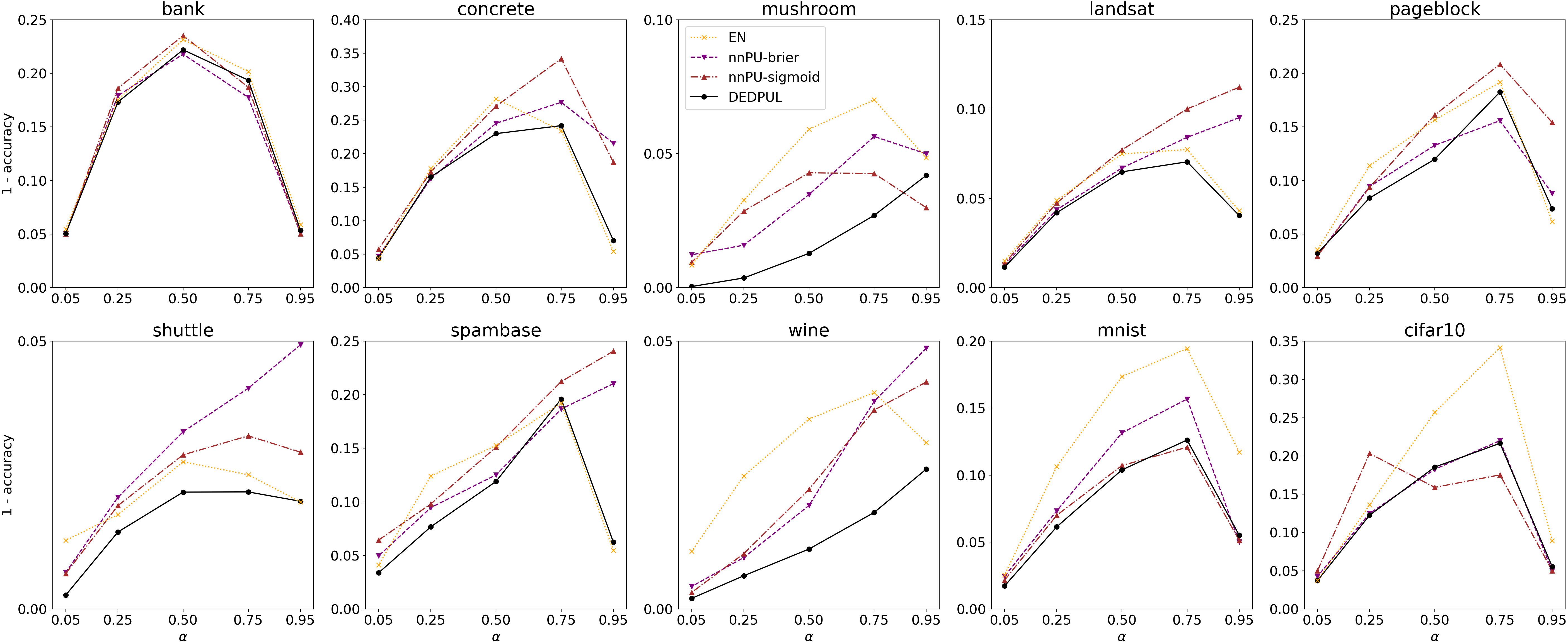}
  \caption{Comparison of Positive-Unlabeled Classification Algorithms, Benchmark Data}
  \label{fig_puc_bench}
\end{figure*}

On some benchmark data sets (Fig. \ref{fig_mpe_bench} -- landsat, pageblock, wine), KM2 performs on par with DEDPUL for all $\alpha$ values except for the highest $\alpha = 0.95$, while on the other data sets (except for bank) KM2 is strictly outmatched. The regions of low proportion of negatives are of special interest in the anomaly detection task, and it is a vital property of an algorithm to find anomalies that are in fact present. Unlike DEDPUL, TIcE, and EN, KM2 fails to meet this requirement. On the synthetic data (Fig. \ref{fig_mpe_synth}) DEDPUL also outperforms KM2, especially when the priors are high.

Regarding comparison with TIcE, DEDPUL outperforms both its vanilla and NTC versions in both synthetic and benchmark settings (Fig. \ref{fig_mpe_synth}, \ref{fig_mpe_bench}). Two insights about TIcE could be noted. First, while the NTC preprocessing is redundant for one-dimensional synthetic data, it can greatly improve the performance of TIcE on the high-dimensional benchmark data, especially on image data. Second, while KM2 generally performs a little better than TIcE on the synthetic data, the situation is reversed on the benchmark data.

In the posteriors estimation task, DEDPUL matches or outperforms both Sigmoid and Brier versions of nnPU algorithm on all data sets in both setups (Fig. \ref{fig_puc_synth}, \ref{fig_puc_bench}). An insight about the Sigmoid version is that it estimates probabilities relatively poorly (Fig. \ref{fig_puc_synth}). This is because the Sigmoid loss is minimized by the median prediction, i.e. either 0 or 1 label, while the Brier loss is minimized by the average prediction, i.e. a probability in $[0, 1]$ range. Conversely, there is no clear winner in the benchmark setting (Fig. \ref{fig_puc_bench}). On some data sets (bank, concrete, landsat, pageblock, spambase) the Brier loss leads to the better training, whereas on other data sets (mushroom, shuttle, wine, mnist) the Sigmoid loss prevails.

Below we show that the performance of DEDPUL can be further improved by the appropriate choice of NTC.

\subsection{Ablation study}\label{section_results_ablation}

DEDPUL is a combination of multiple parts such as NTC, density estimation, EM algorithm, and several heuristics. To disentangle the contribution of these parts to the overall performance of the method, we perform an ablation study. We modify different details of DEDPUL and report changes in performance in Figures \ref{fig_mpe_synth_abl}-\ref{fig_puc_bench_abl}.

\begin{figure*}[t]
\centering
  \includegraphics[width=\linewidth]{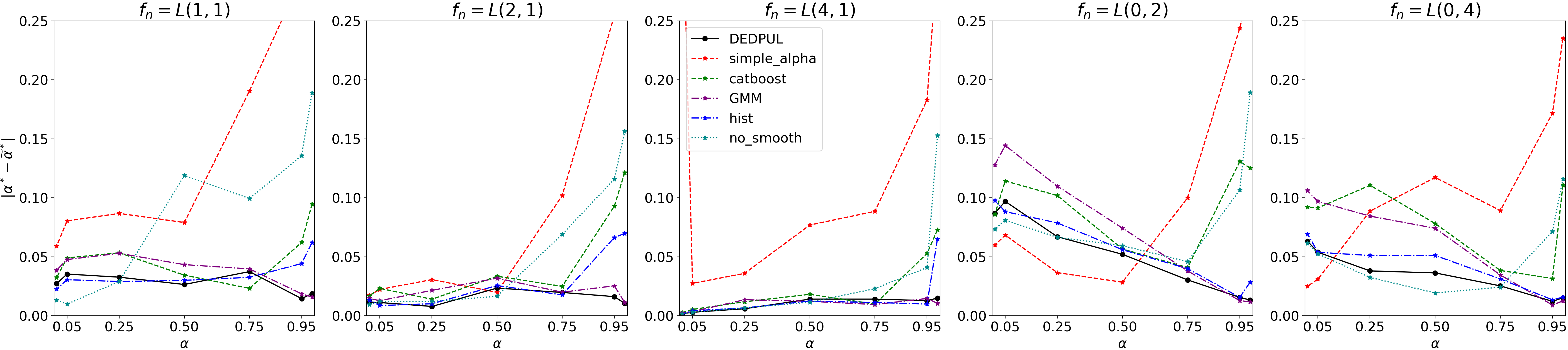}
  \caption{Ablations of DEDPUL, Mixture Proportions Estimation, Synthetic Data}
  \label{fig_mpe_synth_abl}
  
  \includegraphics[width=\linewidth]{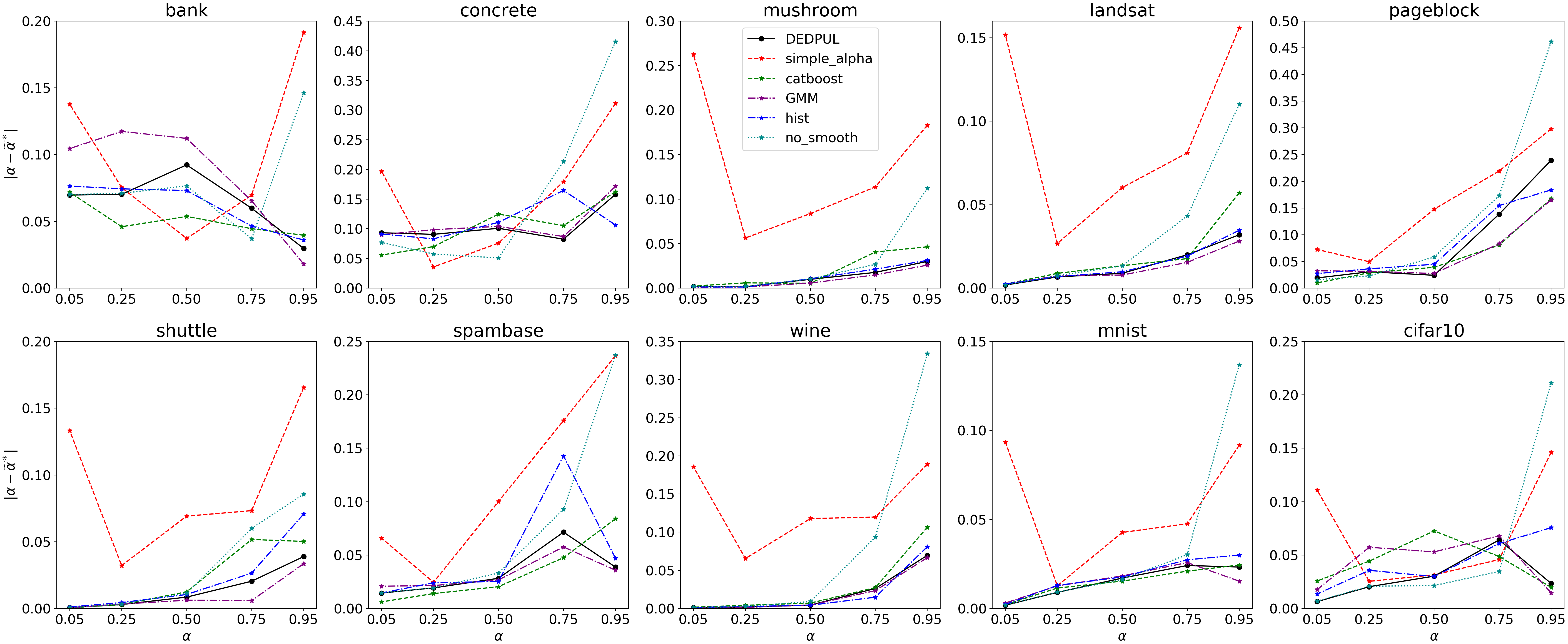}
  \caption{Ablations of DEDPUL, Mixture Proportions Estimation, Benchmark Data}
  \label{fig_mpe_bench_abl}
  
  \includegraphics[width=\linewidth]{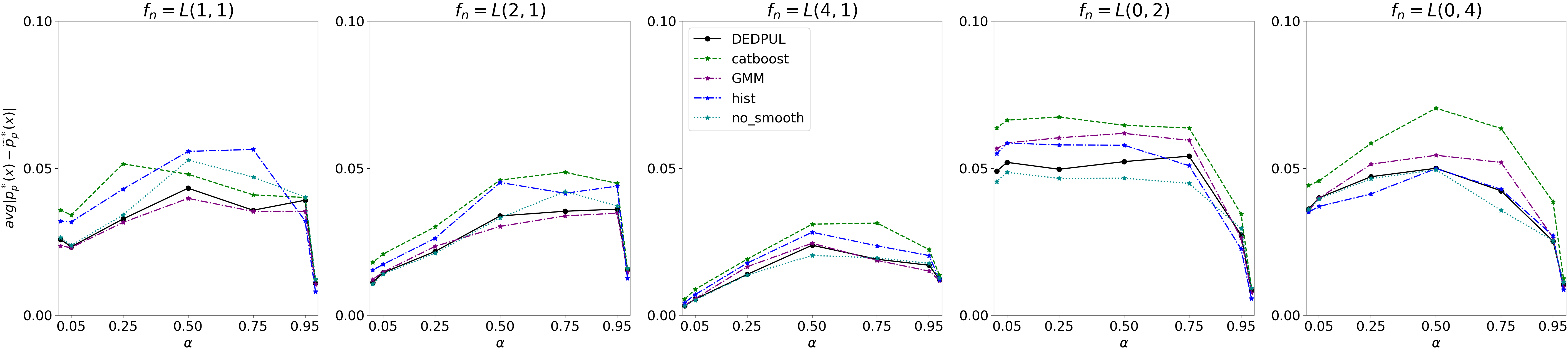}
  \caption{Ablations of DEDPUL, Positive-Unlabeled Classification, Synthetic Data}
  \label{fig_puc_synth_abl}
  
  \includegraphics[width=\linewidth]{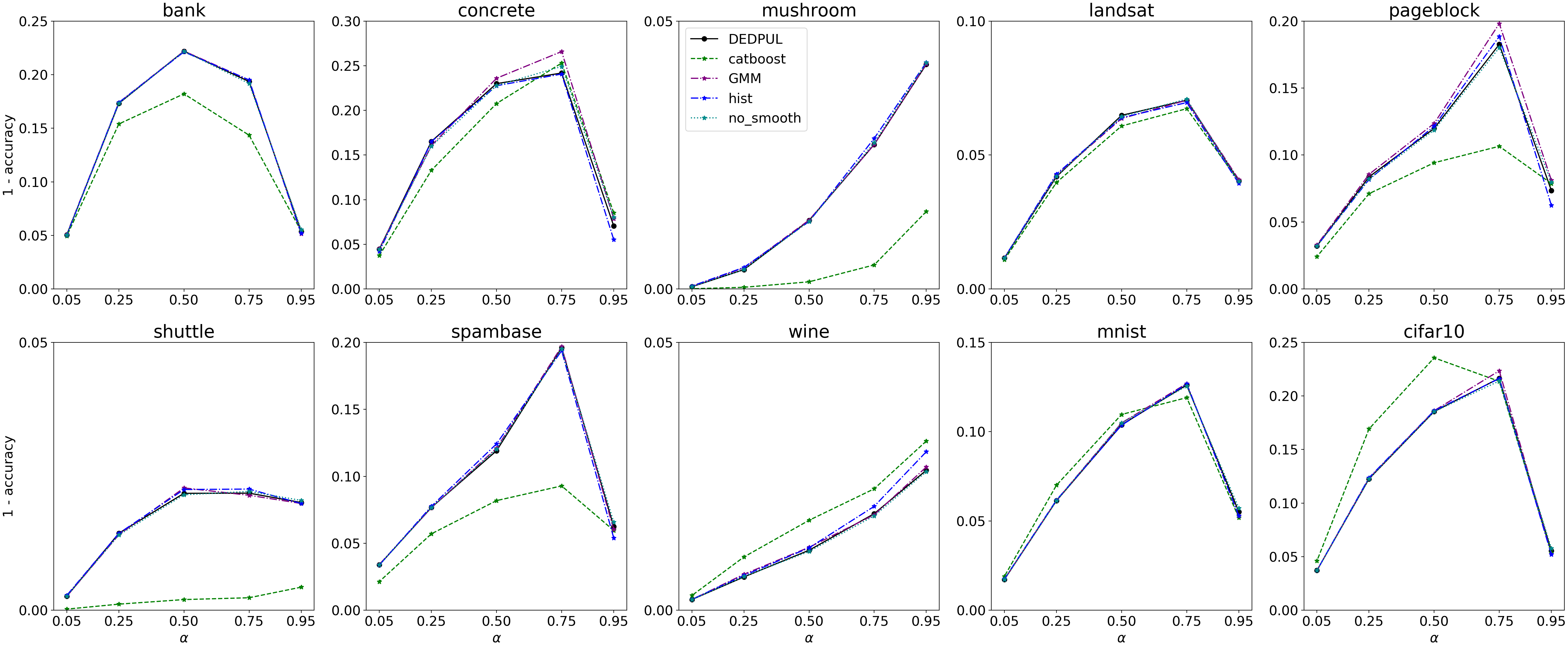}
  \caption{Ablations of DEDPUL, Positive-Unlabeled Classification, Benchmark Data}
  \label{fig_puc_bench_abl}
\end{figure*}

First, we vary NTC. Instead of neural networks, we try gradient boosting of decision trees. We use an open-source library \textit{CatBoost} \cite{catboost}. For the synthetic and the benchmark data respectively, the hight of the trees is limited to 4 and 8, and the learning rate is chosen as 0.1 and 0.01; the amount of trees is defined by early stopping.

Second, we vary density estimation methods. Aside from default kernel density estimation (\textit{kde}), we try Gaussian Mixture Models (\textit{GMM}) and histograms (\textit{hist}). The number of gaussians for GMM is chosen as 20 and 10 for $f_u(x)$ and $f_p(x)$ respectively, and the number of bins for histograms is chosen as 20 for both densities.

Third, we replace the proposed estimates of priors $\alpha_c^*$ and $\alpha_n^*$ with the default estimate (\ref{eq_ntc_preserves_alpha}) (\textit{simple\_alpha}). This is similar to the pdf-ratio baseline from \cite{alphamax}. We replace infimum with 0.05 quantile for improved stability.

Fourth, we remove \textsc{monotonize} and \textsc{rolling\_median} heuristics that smooth the density ratio (\textit{no\_smooth}).

Fifth, we report performance of \textit{EN} method \cite{EN}. This can be seen as ablation of DEDPUL where density estimation is removed altogether and is replaced by simpler equations (\ref{eq_ntc_alpha_star}) and (\ref{eq_ntc_poster_star}).

\paragraph{CatBoost.} In the synthetic setting, CatBoost is outmatched by neural networks as NTC on both tasks (Fig. \ref{fig_mpe_synth_abl}, \ref{fig_puc_synth_abl}). The situation changes in the benchmark setting, where CatBoost outperforms neural networks on most data sets, especially during label assignment (Fig. \ref{fig_puc_bench_abl}). Still, neural networks are a better choice for image data sets. This comparison highlights the flexibility in the choice of NTC for DEDPUL depending on the data at hand.

\paragraph{Density estimation.} On the task of proportion estimation in the synthetic data, \textit{kde} performs a little better than analogues. However, the choice of the density estimation procedure has little effect with no clear trend in other setups (Fig. \ref{fig_mpe_bench_abl}-\ref{fig_puc_bench_abl}), which indicates robustness of DEDPUL to the choice.

\paragraph{Default $\alpha^*$.} On all data sets, using the default $\alpha^*$ estimation method (\ref{eq_ntc_preserves_alpha}) massively degrades the performance, especially when contamination of the unlabeled sample with positives is high (Fig. \ref{fig_mpe_synth_abl}, \ref{fig_mpe_bench_abl}). This is to be expected, since the estimate (\ref{eq_ntc_preserves_alpha}) is a single infimum (or low quantile) point of the noisy density ratio, while the alternatives leverage the entire unlabeled sample. Thus, the proposed estimates $\alpha_c^*$ and $\alpha_n^*$ are vital for DEDPUL stability.

\paragraph{No smoothing.} While not being very impactful on the estimation of posteriors, the smoothing heuristics are vital for the quality of the proportion estimation (Fig. \ref{fig_mpe_synth_abl}, \ref{fig_mpe_bench_abl}).

\paragraph{EN.} In the synthetic setup, performance of EN is decent (Fig. \ref{fig_mpe_synth}, \ref{fig_puc_synth}). It outperforms both versions of TIcE on most and nnPU-sigmoid on all synthetic data sets. In the benchmark setup, EN is the worst procedure for proportion estimation (Fig. \ref{fig_mpe_bench}), but it still outperforms nnPU on some data sets (concrete, landsat, shuttle, spambase) during label assignment (Fig. \ref{fig_puc_bench}). Nevertheless, EN never surpasses DEDPUL, except for L(0, 2).  The striking difference between the two methods might be explained with EN requiring more from NTC, as discussed in Section \ref{section_algorithm_ntc}. 

\section{Conclusion}

In this paper, we have proven that NTC is a posterior-preserving transformation, proposed an alternative method of proportion estimation based on $D(\alpha)$, and combined these findings in DEDPUL. We have conducted an extensive experimental study to show superiority of DEDPUL over current state-of-the-art in Mixture Proportion Estimation and PU Classification. DEDPUL is applicable to a wide range of mixing proportions, is flexible in the choice of the classification algorithm, is robust to the choice of the density estimation procedure, and overall outperforms analogs.

There are several directions to extend our work. Despite our theoretical finding, the reasons behind the behaviour of $D(\alpha)$ in practice remain unclear. In particular, we are uncertain why non-trivial $\alpha_c^*$ does not exist in some cases and what distinguishes these cases. Next, it could be valuable to explore the extensions of DEDPUL to the related problems such as PU multi-classification or mutual contamination. Another promising topic is relaxation of SCAR. Finally, modifying DEDPUL with variational inference might help to internalize density estimation into the classifier's training procedure.

\section*{Acknowledgements}

I sincerely thank Alexander Nesterov, Alexander Sirotkin, Iskander Safiulin,  Ksenia Balabaeva and Vitalia Eliseeva for regular revisions of the paper. Support from the Basic Research Program of the National Research University Higher School of Economics is gratefully acknowledged.

\bibliographystyle{named}
\bibliography{DEDPUL}




\end{document}